%% file: main.tex
\documentclass[12pt]{article}
\usepackage{graphicx}
\usepackage{amsmath,amssymb}
\usepackage{cite}
\usepackage[margin=1.0in]{geometry}

\usepackage[inline,shortlabels]{enumitem} 
\setlist{noitemsep}
\setlist[enumerate]{ labelsep=.25pc, leftmargin=1.5pc } 
\setlist[enumerate,1]{ label= (\arabic*), ref=\arabic*}
\setlist[enumerate,2]{ label= (\roman*),ref  = \roman*}
\setlist[enumerate,3]{ label= (\alph*), ref  = (\alph*)}
\setlist[itemize]{ leftmargin=1.5pc }
\setlist[description]{ font=\sffamily\bfseries }

\usepackage{multirow}
\usepackage{url}
\usepackage{booktabs}
\usepackage{bm}
\usepackage{algorithm}
\usepackage[noend]{algpseudocode}
\renewcommand{\algorithmicrequire}{\textbf{Input:}}
\renewcommand{\algorithmicensure}{\textbf{Output:}}
\DeclareMathOperator*{\argmin}{\operatorname{arg\, min}}

\newcommand{\set}[1]{\{ #1 \}}
\newcommand{\ind}[1]{\mathbb{I}\left[ #1 \right]}

\usepackage{cleveref}
\crefname{algorithm}{Algorithm}{Algorithms}
\crefname{table}{Table}{Tables}
\crefname{figure}{Fig.}{Figures}
\crefname{section}{Sec.}{Sections}
\crefname{chapter}{Chapter}{Chapters}
\crefname{problem}{Problem}{Problems}
\crefname{proposition}{Proposition}{Propositions}

\input{arimmacro_tr.tex} 

\begin{document}
\title{
Computing the Collection of Good Models for Rule Lists
}
%
\author{
Kota Mata\thanks{Presently working for NTT Communications Co.
(e-mail: \texttt{k.mata@ntt.com})
} \and
Kentaro Kanamori\thanks{Presently working for Fujitsu Ltd. 
(e-mail: \texttt{k.kanamori@fujitsu.com})} \and
Hiroki Arimura\thanks{
Graduate School of IST, Hokkaido University, Sapporo 060-0814, Japan. 
  (e-mail: \texttt{arim@ist.hokudai.ac.jp})
} 
}
\date{April 24, 2022}
%
\maketitle              
\begin{abstract}
\input{abst}
\end{abstract}

\input{intro}

\input{prelim}

\input{algo}


\input{prob}

\input{exp}
\input{concl}

%
%
\bibliographystyle{splncs04}
\bibliography{mata,arim}
\end{document}

%% file: arimmacro_tr.tex
\newcommand{\id}[1]{\ensuremath{\mathit{#1}}} 
\newcommand{\idrm}[1]{\ensuremath{\mathrm{#1}}}

\newcommand{\err}{\sharp\idrm{Err}}

\normalfont 

\newcommand{\op}[1]{\texttt{#1}} 
\newcommand{\alg}[1]{\mbox{\sf #1}}

\newcommand{\kw}[1]{\textbf{#1}}

\renewcommand{\vec}[1]{\mathbf{#1}}

\newcommand{\sete}[1]{\{\kern0.00em#1\kern0.00em\}}
\newcommand{\tup}[1]{(\kern0.05em#1\kern0.05em)}

\newcommand{\indicate}[1]{\mathbb{I}\kern-0.1em\left[\kern0.1em{#1}\kern0.1em\right]} 
\newcommand{\eps}{\varepsilon}
\newcommand{\sig}[1]{{\cal #1}}

\newcommand{\indicator}[1]{\textrm{1\kern-0.375em 1}\left[#1\right]}

\newcommand{\name}[1]{{\em #1}}

\newcommand{\textref}[1]{\textit{#1}}
\newcommand{\algoref}[1]{\textref{Algorithm~\ref{#1}}}
\algrenewcommand\algorithmicrequire{\textbf{Input:}}
\algrenewcommand\algorithmicensure{\textbf{Output:}}
\makeatletter
\renewcommand{\ALG@name}{Algorithm\hspace{-0.2em}}
\makeatother
\newcommand{\Head}{\Statex\hspace{-1.0\leftmargin}}

\newcommand{\Proc}{\Head\kw{Procedure}\hspace{0.2em}}

\newtheorem{theorem}{Theorem}

\newtheorem{definition}{Definition} 
\newtheorem{problem}{Problem} 

\def\lowerbox{\raise-.12ex\hbox{$\Box$}}
\newcommand{\qed}{\hfill{\lowerbox}\vspace{0.8\topsep}\par}
\newenvironment{proof}{
  \vspace{-0.5\topsep}\noindent{\textit{Proof}: \rule{.15em}{0mm}}
}{\qed } 

\newsavebox{\cmbox}
 


%% file: abst.tex
Since the seminal paper by Breiman in 2001, who pointed out a potential harm of prediction multiplicities from the view of explainable AI, global analysis of a collection of \name{all good models}, also known as a ``\name{Rashomon set},'' has been attracted much attention for the last years. Since finding such a set of good models is a hard computational problem, there have been only a few algorithms for the problem so far, most of which are either approximate or incomplete. 
To overcome this difficulty, we study efficient enumeration of all good models for a subclass of interpretable models, called rule lists. 
Based on a state-of-the-art  optimal rule list learner, \text{CORELS}, proposed by Angelino \textit{et al.} in 2017, we present an efficient enumeration algorithm \textit{CorelsEnum} for exactly computing a set of all good models using polynomial space in input size, 
given a dataset and a error tolerance from an optimal model. 
By experiments with the COMPAS dataset on recidivism prediction, our algorithm \textit{CorelsEnum} successfully enumerated all of several tens of thousands of good rule lists of length at most $\ell = 3$ in around 1,000 seconds, while a state-of-the-art top-$K$ rule list learner based on Lawler's method combined with \text{CORELS}, proposed by Hara and Ishihata in 2018, found  only 40 models until the timeout of 6,000 seconds. 
For global analysis, we conducted experiments for characterizing the Rashomon set, and observed large diversity of models in predictive multiplicity and fairness of models. 

%% file: intro.tex

\section{Introduction}
In applications of machine learning models to critical decision-making tasks, such as judicial decisions and loan approvals, there have been increasing concerns about the \emph{interpretability} of the models~\cite{Guidotti:CSUR,Rudin:NMI2019}. 
If the decisions based on their predictions might have a significant impact on individuals, decision-makers must provide the reason of the decisions to assure users of their correctness~\cite{Rudin:NMI2019}. 
Consequently, learning interpretable models, such as decision trees, rule sets, and rule lists, has attracted considerable attention in recent years~\cite{hastie2001eslbook,Guidotti:CSUR,angelino:rudin:kdd2017corels,hara:ishihata:aaai2018rulemodels}. 
Because these models are expressed as
combinations of 
simple ``if-then" rules 
as shown in \cref{table:rulelists}, 
it is easy for humans to understand and validate how the models make predictions~\cite{Guidotti:CSUR}. 

Recently, for interpretable models, there has been another concern about the situation where there exist multiple models that are approximately equally accurate by relying on different features~\cite{semenova:rudin:parr:arxiv2019rashomon,Marx:ICML2020,Hancox-Li:FAT*2020}. 
In the seminal 
paper, Breiman~\cite{breiman:statsci2001twocultures} has named such phenomenon ``\emph{Rashomon effect}". 
By showing examples of feature importance, he 
explained 
how different models with similar accuracy can generate different explanations for prediction tasks. From this view, he argues that it is unreliable to use explanations derived from a single predictive model for the class of interpretable models such as decision trees and rule lists. 
By introducing the notion of \emph{prediction multiplicities}, Marx {\em et al.} ~\cite{Marx:ICML2020} showed how a prediction problem can show multiplicities, and how we can measure the diversity of a set of good models. 

For example, we show in \cref{table:rulelists} a part of  results of experiments in \cref{sec:exp} on the COMPAS dataset~\cite{compas:2016} for the task of predicting two-year recidivism. 
The table contains a pair of competing rule lists which was found by our algorithm \alg{CorelsEnum}, associated with the values of the accuracy ($\mathrm{Acc}$), and two major discrimination measures, namely, \emph{demographic parity} ($\mathrm{DP}$)~\cite{Calders:ICDMWS2009} and \emph{equal opportunity} ($\mathrm{EO}$)~\cite{Hardt:NIPS2016} (see \cref{sec:prob:scores}).
Although two rule lists have similar accuracies of 62.5\% and 60.9\%, respectively,  they have quite different characteristics in DO and EO. 
Moreover, we observed that there were some rule list which was only 1\% less accurate than an optimal rule list, while it made different predictions on 11\% of training data from the optimal one made. This type of prediction multiplicity is called \emph{discrepancy}~\cite{Marx:ICML2020}, and will be discussed later in \cref{sec:exp:multiplicities}. 



\begin{table}[t]
 \caption{%
 An example of a pair of competing rule lists of length $\ell = 3$ with similar accuracies, 62.5\% and 60.9\%, for predicting  two-year recidivism on the COMPAS dataset. Although two rule lists have similar accuracy (Acc), they have quite different values of discrimination measures, namely, demographic parity (DP) and equal opportunity (EO), whose definitions can be found in \cref{sec:prob:scores}.  They also have a large discrepancy value $0.325$ (the relative Hamming distance between their prediction vectors).
  }
 \label{table:rulelists}
\medskip
\def\tabcolsep{.5em}
\begin{tabular}{cccl}
\hline 
$\mathrm{Acc}$
& $\mathrm{DP}$
& $\mathrm{EO}$
& \multicolumn{1}{c}{Rule list models}
\\
\hline \hline 
0.625
& 0.083
& 0.061
&  \rule[-6mm]{0.0em}{14mm}
$R_1:$ \scalebox{1.0}{
\begin{tabular}{l}\tt 
   if juvenile-felonies$>$0 \& \\ \tt \qquad current-charge-degree=Felony, then Yes\\ \tt 
   else if juvenile-misdemeanors=0 \& \\ \tt \qquad priors$>$3, then Yes\\  \tt 
   else predict No 
\end{tabular}
}
\\ \hline
0.609
& 0.052
& 0.042
& \rule[-6mm]{0.0em}{14mm}
$R_2:$
\begin{tabular}{l} \tt 
   if sex=Male \& juvenile-crimes$>$0, then Yes\\ \tt 
   else if age=18-20 \& priors=0, then Yes\\ \tt 
   else predict No
\end{tabular}
\\
\hline 
\end{tabular}
\end{table}

The central notion in the studies mentioned above is the \name{collection of good models} within a given model class $\mathcal{H}$ 
that have similar accuracy as an optimal model on a given dataset, which is also called a ``\emph{Rashomon set}", and  has been discussed by several authors~\cite{Marx:ICML2020,semenova:rudin:parr:arxiv2019rashomon}. 
Here, we assume to measure the goodness of a model $h$ by the \emph{empirical risk} $L(h)$ on dataset $S$, which is the 
proportion of the data that the model makes incorrect predictions.
Then, the notion of Rashomon sets is captured by the following definition, 
due to  
 Fisher, Rudin, and Dominici~\cite{Fisher:JMLR2019}: 
the \name{Rashomon set} with error tolerance $\varepsilon > 0$ is defined as the set $\sig R_\varepsilon$ of all models $h$ whose empirical risk $L(h)$ is at most larger than that of optimal model $h_*$ within tolerance $\varepsilon > 0$, that is, given by: 
\begin{align*}
    \sig R_\varepsilon := \set{h \in \mathcal{H} \mid L(h) \leq L(h^*) + \varepsilon},
\end{align*}
Although all models in $\sig R_\varepsilon$ achieve similar accuracy, they often differ markedly in their predictions for individual inputs and thus may have different properties~\cite{Rudin:NMI2019,Marx:ICML2020,Hancox-Li:FAT*2020}. 
Consequently, characterizing 
the set 
$\sig R_\varepsilon$ plays an important role in validating the reliability of 
$\mathcal{H}$ on a specific prediction problem~\cite{Marx:ICML2020}.

To characterize the Rashomon set $\sig R_\varepsilon$ by existing criteria, one often needs to compute the set $\sig R_\varepsilon$ for a certain model class $\mathcal{H}$ on a given dataset. 
However, 
since $\sig R_\varepsilon$ can contain exponentially many models in the input size, 
exact computation of $\sig R_\varepsilon$ still remains challenging~\cite{Rudin:NMI2019}. 
Although there are only a few existing methods for the task~\cite{semenova:rudin:parr:arxiv2019rashomon,hara:ishihata:aaai2018rulemodels}, 
they can only provide a subset of $\sig R_\varepsilon$ randomizedly or approximately. 
Therefore, no one has exactly computed the Rashomon set $\sig R_\varepsilon$ for the class of interpretable models on real datasets and measured the existing criteria to characterize the set $\sig R_\varepsilon$~\cite{Rudin:NMI2019}.

In this paper, we focus on the class of
\emph{rule lists}~\cite{hara:ishihata:aaai2018rulemodels,angelino:rudin:kdd2017corels}, 
and study an exact computation of all the rule lists in the Rashomon set $\sig R_\varepsilon$. 
For that purpose, we extend \emph{CORELS}~\cite{angelino:rudin:kdd2017corels}, 
which is a state-of-the-art optimal rule list learner, 
and propose an efficient algorithm for exactly computing 
the set $\sig R_\varepsilon$
on a given dataset and 
the best-achievable empirical risk.
Based on $\sig R_\eps$, we then measure the following prediction multiplicity scores~\cite{Marx:ICML2020}:  the \emph{ambiguity} $\alpha_\eps$ is the proportion of data that has at least one model with conflicting prediction from $h_0$, while the \emph{discrepancy} $\delta_\eps$ is the maximum proportion of data that a model can make different prediction from $h_0$ over all good models (see \cref{sec:prob:scores}).

Our contributions are summarized as follows:
\begin{itemize}
    \item
    We propose an exact algorithm \textsf{CorelsEnum} for computing the Rashomon set for the class of rule lists. 
    Based on CORELS~\cite{angelino:rudin:kdd2017corels}, 
    our algorithm can efficiently enumerate all good rule lists with length at most $K$ and within error tolerance $\eps$. 
    Unlike the previous method~\cite{hara:ishihata:aaai2018rulemodels}, \textsf{CorelsEnum} uses only polynomial working space to compute the whole set.
    
    \item
    By experiments on the COMPAS dataset~\cite{compas:2016},with a large value of $\eps = 15\%$, 
    our \textsf{CorelsEnum} successfully 
    computed the Rashomon set $\sig R_\eps$ of 
    around 23,354 all good rule lists of length at most $\ell = 3$ in 1,000 seconds, while the previous one for top-$K$ rule lists, \textsf{CorelsLawler}~\cite{hara:ishihata:aaai2018rulemodels}, listed only top-40 rule lists before the timeout of 6,000 seconds. 
    
    \item 
    Based on the computed Rashomon sets $\sig R_\eps$, 
    we analyzed the diversity of a set of good models 
    in terms of 
    \emph{predictive multiplicity}~\cite{Marx:ICML2020} and \emph{unfairness range}~\cite{Coston:ICML2021,Aivodji:NIPS2021}. 
    We found that 
    the Rashomon set $\sig R_\eps$
    with small error tolerance 
    $\eps = 1\%$ had large prediction multiplicities 
    $\alpha_\eps = 29\%$ and $\delta_\eps = 11\%$. 
    For discrimination scores, we observed 
    a trade-off between the score and the empirical risk, 
    and the existence of a few clusters of good models with similar scores. 
    \end{itemize}


As consequences, our results 
revealed 
that real datasets such as COMPAS could had the large diversity of models that cannot be ignored in explanability. Thus, we need further researches for efficient methods to integrate competitive rules to apply existing model explanation methods. 

\subsection{Related Work}

Rule models, such as decision trees, rule sets, and rule lists,
are popular \name{interpretable models}~\cite{Lakkaraju:KDD2016,angelino:rudin:kdd2017corels,Guidotti:CSUR,Rudin:NMI2019}. 
Among them, \name{rule lists} and their variants~\cite{angelino:rudin:kdd2017corels,Wang:AISTATS2015,hara:ishihata:aaai2018rulemodels} have been widely studied from the view of global optimization. 
Angelino~{\em et al.}~\cite{angelino:rudin:kdd2017corels} proposed an algorithm \alg{CORELS} that finds a single optimal rule list that exactly minimizes the size-penalized empirical risk by branch-and-bound search. 
In this paper, we extended \alg{CORELS} for computing the complete set of all almost-accurate rule lists using enumeration and data mining techniques~\cite{Han:kamber:pei2011dmbook}. 



\emph{Computation of the Rashomon set} $\sig R_\eps$ 
has been attracting increasing attention in recent years~\cite{Rudin:NMI2019} from various perspectives, such as interpretability~\cite{Fisher:JMLR2019,semenova:rudin:parr:arxiv2019rashomon}, predictive multiplicity~\cite{Marx:ICML2020}, and fairness~\cite{Coston:ICML2021,Aivodji:NIPS2021}. 
However, exact computation of $\sig R_\eps$ with a small memory footprint
still remains challenging~\cite{Rudin:NMI2019}. 
Particularly, Semenova {\em et al.}~\cite{semenova:rudin:parr:arxiv2019rashomon} described a procedure for randomly sampling a subset of $\sig R_\eps$ for decision trees of bounded size. 
Hara and Ishihata~\cite{hara:ishihata:aaai2018rulemodels} have proposed an efficient top-$K$ rule list learner, called \alg{CorelsLawler} here, 
based on empirical risk 
using the well-known Lawler's method~\cite{hara:ishihata:aaai2018rulemodels}. 
We remark that neither of the above methods did not achieve as goals exact computation of the whole $\sig R_\eps$ and polynomial working space. In contrast, our algorithm achieved both of these requirements. 



%% file: prelim.tex
\section{Preliminaries}
\label{sec:prelim}

In this section, we give basic definitions and notation, which will be necessary in the following sections. We also introduce our problem of computing the collection of all good models for a class of models. For the notions that are not found here, please consult appropriate textbooks 
such as~\cite{hastie2001eslbook}. 

\subsection{Notation}

For a predicate $\psi$, $\ind{\psi}$ denotes the indicator of $\psi$; that is, $\ind{\psi}=1$ if $\psi$ is true, and $\ind{\psi}=0$ otherwise.
Throughout this paper, we consider the \emph{binary classification problem} as our prediction problem, 
and assume Boolean features 
as in most studies on learning rule models~\cite{angelino:rudin:kdd2017corels,Lakkaraju:KDD2016}. 
Then, the input and output domains are 
$\mathcal{X} = \set{0,1}^{J}$ 
and $\mathcal{Y} = \set{0, 1}$, respectively, where $J \in \mathbb{N}$ is the number of features.
An \emph{example} is a tuple $({\vec x}, y)$ of an input vector (or an \name{input}) ${\vec x} = (x_1, \dots, x_J) \in \mathcal{X}$ and a prediction label (or a \name{label}) $y \in \mathcal{Y}$, and a \emph{dataset} is a sequence 
$S = \set{({\vec x}_n, y_n)}_{n = 1}^{N}$ 
of $N$ examples, where $S \in (\sig X\times \sig Y)^N$. 
For a given \name{classifier}, or a \name{prediction model}, 
$h \colon \mathcal{X} \to \mathcal{Y}$
and dataset $S$, the \emph{empirical risk} of $h$ is defined as
\begin{math}
  L(h \mid S) := \frac{1}{N} {\sum}_{n=1}^{N} l(y_n, h({\vec x}_n)) \in [0,1],
\end{math}
where $l \colon \mathcal{Y} \times \mathcal{Y} \to \mathbb{R}_{\geq 0}$ is a \emph{loss function} that measures the difference between the prediction $h({\vec x})$ and the true label $y$. 
In this paper, we assume the $0$-$1$ loss $l(y, \hat{y}) = \ind{y \not= \hat{y}}$. 
The \name{number of misclassifications} by $h$ on $S$ is defined as $\err(h \mid S) := \sum_{n=1}^{N} l(y_n, h({\vec x}_n)) \in [0..N]$. 
Note that the empirical risk is given by $L(h \mid S) = \frac{1}{N} \err(h \mid S)$.

\subsection{Rule List}

In this study, we focus on the class of classifiers, called \emph{rule lists}~\cite{angelino:rudin:kdd2017corels,hara:ishihata:aaai2018rulemodels}, defined as follows.
Let $\sig X = \set{0,1}^J$ be an input domain of $J$ Boolean features. 
Let $\sig T$ be 
a set, called a \name{vocabulary}, 
which consists of terms over a set of $J$ \name{Boolean features} $x_1, \dots, x_J$
over $\set{0,1}$.
Each \name{term}  $t$ in $\sig T$ is a conjunction
$t = (x_{i_1}\land \dots \land x_{i_k})$
of Boolean features, and represents a Boolean assertion $t \colon \sig X \to \set{0,1}$ such that $t$ evaluates \name{true} on an input vector
$\vec x \in \sig X$ 
if $x_{i_j}=1$ for all $1\le j\le k$,
and \name{false} otherwise.
For example, $(\text{`age = 18 - 20'}) \land (\text{`sex = Male'})$ is a term used in experiments of \cref{sec:exp}.
As with previous studies~\cite{angelino:rudin:kdd2017corels,Lakkaraju:KDD2016}, we assume that $\sig T$ includes the constant $1$ (true),  and that $\sig T$ is pre-mined by frequent itemset mining algorithms (e.g., FP-growth~\cite{Han:kamber:pei2011dmbook} or LCM~\cite{uno2004lcmver2}) 

Let $\sig Y$ be a set of prediction labels. A \name{rule} over $\sig T$ and $\sig Y$ is a pair $(t \to y)$ of a term $t \in \sig T$ and a label $y \in \sig Y$, which corresponds to the conditional statement ``if $t$, then $y$.''
A \name{rule list} of length $\ell\ge 1$ over $\sig T$ and $\sig Y$ is a tuple 
$d = (r_1, \dots, r_\ell)$
of $\ell$ rules, where
(i) $r_i = (t_i \to y_i)$
is a rule for every $1\le i\le \ell$, and
(ii) the last rule $r_\ell$ always has constant test $t_\ell = 1$, and is called the \name{default rule}.
In \cref{table:rulelists}, we show an example of a rule list.
We denote by $\circ$ the concatenation operation for rule sequences. 
A rule list $d = ((t_i\to y_i))_{i=1}^{\ell}$
naturally defines a \name{prediction model} $h_d \colon \mathcal{X} \to \mathcal{Y}$ such that
given an input $\vec x$ in $\sig X$,
the \name{prediction} $y = h_d(\vec x)$ in $\sig Y$ is computed by the code below: 
\begin{itemize}
\item[]\tt
  if $t_1(\vec x)$ then predict $y_1$,
  else if $t_2(\vec x)$ then predict $y_2$,
  $\dots$, 
  \par else if $t_{\ell-1}(\vec x)$ then predict $y_{\ell-1}$, 
  else predict $y_\ell$. 
\end{itemize}

In the above code, whenever the label $y_i$ is predicted, the condition $t_1(\vec x)=0\land \dots\land t_{i-1}(\vec x)=0$ and $t_i(\vec x)=1$ must hold. Then, we say that $\vec x$ \name{falls into} the $i$-th rule $r_i$. 
For a given dataset 
$S \in (\sig X\times \sig Y)^N$,
regularization 
parameter $\lambda \geq 0$,
and a set $\sig T$ of candidate terms, 
the task of learning a rule list is formulated 
as: 
\begin{align}
  h_{d^*}
  = {\argmin}_{h_d \in \mathcal{H}_{\sig T}} R_{\lambda}(h_d \mid S) := L(h_d \mid S) + \lambda \cdot |d|.
  \label{eq:corels}
\end{align}
Although finding an optimal solution of the problem~\eqref{eq:corels} is a
hard
combinatorial optimization, 
it can be efficiently solved by recent branch-and-bound optimization algorithms such as {CORELS}~\cite{angelino:rudin:kdd2017corels} in many practical instances.

\subsection{Computation of Rashomon Sets}

To characterize the set of good models, the \emph{Rashomon set} has been introduced as a set of models that achieve near-optimal accuracy~\cite{Rudin:NMI2019}. 
For a prediction problem $(\mathcal{X}, \mathcal{Y})$, let $\mathcal{H}$ be a set of classifiers $h \colon \mathcal{X} \to \mathcal{Y}$, which we call a \emph{model class}. 
Following previous studies~\cite{semenova:rudin:parr:arxiv2019rashomon,Marx:ICML2020}, we define the Rashomon set as a subset of classifiers that achieve accuracy close to a given \emph{reference classifier} $h_0 \in \mathcal{H}$ with respect to a certain loss function $l$ and a given error tolerance $\varepsilon \geq 0$. 

\begin{definition}
    Given a model class $\mathcal{H}$, reference classifier $h_0 \in \mathcal{H}$, dataset $S$, and error tolerance $\varepsilon \geq 0$, the \emph{Rashomon set} $\sig R_{\varepsilon}(h_0 \mid S)$ is defined as follows:
    \begin{align*}
        \sig R_{\varepsilon}(h_0 \mid S) := \set{ h \in \mathcal{H} \mid L(h \mid S) \leq L(h_0 \mid S) + \varepsilon}. 
    \end{align*}
\end{definition}

As with existing studies~\cite{semenova:rudin:parr:arxiv2019rashomon,Marx:ICML2020}, we assume the reference classifier $h_0$ to be an optimal rule list $h_{d^*}$ for the learning problem~\eqref{eq:corels}, which can be obtained using {CORELS}~\cite{angelino:rudin:kdd2017corels}. 
Note that the choice of $h_0$ is independent of our results. 
Now, we formally define our problem as follows: 

\begin{problem}\label{prob:rashomonrl}
    Given a dataset $S$, a set of terms $\sig T$, a reference rule list $h_0$, an error tolerance $\varepsilon \geq 0$, and $\ell \geq 0$, compute the Rashomon set $\sig R_{\varepsilon}(h_0 \mid S)$
    for the class of rule lists of length at most $\ell$. 
\end{problem}

By solving \cref{prob:rashomonrl}, we can obtain the Rashomon set $\sig R_{\varepsilon}(h_0 \mid S)$ of rule lists of length $\le \ell$, and can analyze the properties of $\sig R_{\varepsilon}(h_0 \mid S)$ 
from various perspectives described in Sec.~4.

\subsection{Optimal Rule List Learner \alg{CORELS}}


Our algorithm is designed based on the recent branch-and-bound optimization algorithm \alg{CORELS} for learning a single optimal rule list, proposed by~Angelino {\em et al.}~\cite{angelino:rudin:kdd2017corels}. Here, we will briefly review \alg{CORELS}, and discuss how we can extend \alg{CORELS} to exact computation of the Rashomon set. 

The inputs to the \alg{CORELS} algorithm are 
a set $\sig T$ of terms, a set $\sig Y$ of labels, 
a training dataset $S$, 
and numbers $\ell\ge 1$ and $\lambda > 0$. 
Invoked as 
$\alg{Corels}(
\sig T,  
\sig Y, 
\ell, 
\lambda, 
S
)$ with input parameters, 
\alg{CORELS} 
finds an optimal rule list $d_*$ with length $\le \ell$ 
that minimizes the objective $R(d_*)$ in \cref{eq:corels}
by traversing the hypothesis space of prefixes of rule lists as follows. 
For every $1\le k\le \ell$, let $d_{k} := r_1 \circ \dots \circ r_k$ is called a $k$-\name{prefix}, 
where $r_i = (t_i \to y_i)$ and $\circ$ is the concatenation. 
Then, \alg{CORELS} starts with the empty prefix $()$ and 
by recursively expanding the current $(k-1)$-prefix $\pi$ to $k$-prefix 
$\pi' = \pi\circ r_k$, $0\le k\le \ell$, by appending a new rule
$r_k\in \sig T\times\sig Y$. 
The \alg{CORELS} algorithm employs sophisticated pruning strategies using constraints such as maximum rule length $L$,
and the estimate of a lower bound of the objective. %

If 
$M_\textrm{corels} \le |\sig Y|^{\ell}|\sig T|^{\ell-1}$ 
is the number of caldidate prefixes for \alg{CORELS} to visit, 
\alg{CORELS} runs in 
$t_\textrm{corels} = O(M_\textrm{corels} |S|)$ 
time and 
$s_\textrm{corels} = O(L + |S| + |\sig T|)$ 
space in the worst case using stack of length at most $L$. 


  
  
  


%% file: algo.tex
\section{Methods for finding good models}
\label{sec:algo}


In this section, we study efficient methods for finding a set of good models on a given training dataset.
Firstly, in Sec.~\ref{sec:algo:lawler:with:corels}, we briefly review an existing algorithm, referred to as \alg{CorelsLawler} in this paper, for Top-$K$ enumeration of good rule lists using \alg{CORELS} algorithm as a black-box function, proposed by Hara and Ishihata~\cite{hara:ishihata:aaai2018rulemodels}. 
Next, in Sec.~\ref{sec:algo:proposed}, we propose  our 
algorithm \alg{CorelsEnum} that efficiently enumerate all the rule lists of length $\le K$ in the Rashomon set on a given dataset. 
    

\subsection{Lawler's method combined with \alg{Corels} algorithm}
\label{sec:algo:lawler:with:corels}

Lawler's method~\cite{Lawler:MS1972} is a well-known framework for top-$K$ enumeration using a black-box optimization function. In \cref{algo:lawler:kbest}, we show the pseudo-code for Hara and Ishihata's algorithm~\cite{hara:ishihata:aaai2018rulemodels}, called \alg{CorelsLawler} here, for finding top-$K$ rule lists using Lowler's method. This algorithm iteratively calls \alg{CORELS}~\cite{angelino:rudin:kdd2017corels}, to find one of the optimal rule lists within the subspace of hypothesis. During the search, It removes some terms appearing in a discovered rule list $Rule$ from $\sig T$ to efficiently search the hypothesis space of good models, where $Rule.\op{Terms}()$ is the set of terms appearing in $Rule$. 

If $K$ is the number of good models to output, we can show that the time and space complexity of \alg{CorelsLawler} is
at most $t_\mathrm{lawler} = O( t_{corels}\cdot K \ell)$ time and $s_\mathrm{lawler} = O(s_\mathrm{corels} + K \ell)$ space. 
A major disadvantage of \alg{CorelsLawler} is its exponential space complexity 
since it must keep the set $\sig F$ of all terms found so far for the membership test at Line~7. 
Since $|\sig F| \le K \le M_{\rm corels} \le |\sig T|^{\ell-1}|\sig Y|^{\ell}$, $|\sig F|$ becomes exponential in $\ell$ in the worst case. 



\begin{algorithm}[t]
  \caption{
    Lawler's method with CORELS for finding Top-$K$ rule lists with respect to prediction error (score). 
  }\label{algo:lawler:kbest}
\begin{algorithmic}[1]
\Require A set $\sig T$ of all terms, a label set $\sig T$, $\ell \ge 0$, $\lambda > 0$, and a dataset $S$. 
\Ensure A list $Answers$ of top-$K$ rule lists in prediction error.
\Proc \alg{CorelsLawler}
\State $Answers \gets \emptyset$
\State 
$(score, Rule) \gets \alg{Corels}(\sig T, \sig Y, \ell, \lambda, S)$

\State $Queue \gets \set{ (score, (Rule, \sig T, \emptyset)) }$
\Comment{A priority queue of $(Rule, T, F)$ with $score$ as key, where $Rule$ is a rule set, $T$ and $F$ are include and exclude sets of features.}
\While{$Queue \not= \emptyset$ and $|Answers| < K$}
  \State $(score, (Rule, T, F)) \gets Queue.\op{deletemin}()$
  \Comment{An entry with minimum $score$}
  \State $\id{Terms} \gets Rule.\op{Terms}()$
  \If{$\id{Terms} \not\in \sig F$}
    \Comment{$\id{Terms}$ is the set of all terms used in Rule}
    \State $Answers \gets Answers \cup \set{ (score, Rule) }$
    \State $\sig F \gets \sig F\cup\set{ \id{Terms} }$
  \EndIf
  \For{each $f \in \id{Terms}$}
    \State $(score', Rule') \gets \alg{Corels}((T\!\setminus\!\set{ f }),  \sig Y, \ell, \lambda, S)$
    \State $Queue \gets Queue \cup\set{ (score', (Rule', (T\!\setminus\!\set{f}), F\cup\set{ f }) }$
  \EndFor
\EndWhile
\State \textbf{return} $Answers$
\end{algorithmic}
\end{algorithm}

\subsection{The Proposed Algorithm \alg{CorelsEnum}}
\label{sec:algo:proposed}

By extending CORELS, 
we devised our algorithm $\alg{CorelsEnum}$ for computing the Rashomon set of rule lists in polynomial space in $\ell$ and other inputs. 
In \algoref{algo:reccorel:basic}, we show the pseudo-code of the $\alg{CorelsEnum}$ algorithm. 
Given a vocabulary $\sig T$, a label set $\sig Y$, the maximum length parameter $\ell\ge 0$, a dataset $S$ of $N$ example, and the empirical risk $L(h_{d_0} \mid S)$ of a reference rule list $d_0$, 
$\alg{CorelsEnum}$ traverses the space of rule lists in depth-first manner from a shorter prefix to longer one,  starting from the empty prefix $()$. 

At each iteration with a candidate prefix $dp = (r_1, \dots, r_k)$, $0\le k\le \ell$, the algorithm either builds a rule list $d$ from the current prefix $dp$, or makes branching with children $dp' = dp \circ (t \to y)$ for all possible combinations of a term $t$ in $\sig T$ and a label $y$ in $\sig Y$. 


Invoked with as arguments 
$dp=()$, $k$, $L_{*} = L(h_0 \mid S) + \varepsilon$, $\sig T$, $\sig Y$, $\ell$, and $S$, 
the recursive procedure \alg{CorelsEnum} computes the Rashomon set of all rule lists with length $\le \ell$
on a dataset $S$ at each iteration as follows: 
\begin{itemize}[$\bullet$]
\item Receive the current candidate prefix $dp$ of length $0\le k\le \ell$ over $\sig T$. 
  
\item For each label $y$ in $\sig Y$, test if the rule list $d = dp \circ (1\to y)$ and its empirical risk $L = L(h_d \mid S)$ satisfies that $L < L_*$. If the test succeeds, output the pair $(d, L)$ as a solution. 
  
\item For each $t \in \sig T$ and $y \in \sig Y$, do:
  First, generate the child prefix $dp' = dp\circ (t\to y)$  of length $k+1$ from $dp$ by appending a new rule $(t\to y)$, make a recursive call with $dp'$, and updating $\sig T'$ by removing $t$ to avoid duplicates. 
\end{itemize}


\begin{algorithm}[t]
  \caption{
    A basic algorithm $\alg{CorelsEnum}$ for computing the Rashomon set $\sig R_\eps(h_0 \mid S)$ consisting of all rule lists $h_d$ with length $\le \ell$ 
    such that $L(h_d \mid S) \le L(h_{d_0} \mid S) + \eps$,
    with respect to a reference rule list $h_{d_0}$. 
  }\label{algo:reccorel:basic}
\begin{algorithmic}[1]
\Proc $\alg{CorelsEnum}(dp, k, L_*, \sig T, \sig Y, \ell, S)$:
\Require
    A candidate prefix $dp = (r_1, \dots, r_k)$, 
    its length $k\ge 0$,
    a non-empty set of terms $\sig T$,
    a label set $\sig Y$, 
    $\ell \ge 0$, $\lambda > 0$, $L_{*} \in [0,1]$, 
    and a dataset $S\in (\sig X\times \sig Y)^N$. 
\Ensure The subset of $\sig R_\eps(h_0 \mid S)$ consisting of all rule lists 
with prefix $dp$. 

  \For{label $y \in \sig Y$}
    \Comment{Step 1: Processing a rule list $d$ with default label $y$}
    \State $d \gets (dp \circ (1 \to y))$; 
    $L \gets L(h_d \mid S)$ 
    \If{$L \le L_*$ }
      \State \kw{Output} $(d, L)$ as a solution
      \Comment{A solution is found}
    \EndIf
  \EndFor
  \State \kw{if}{$k \ge \ell$} \kw{then return}
\Comment{Pruning by the maximum length}  
  \For{term $t \in \sig T$}
    \Comment{Step 2: Generating children of a parent prefix $dp$}
    \For{label $y \in \sig Y$}
      \State $dp' \gets\: (dp\circ (t \to y))$
      \If{$LB(dp', S) \le L_*$}
      \Comment{Pruning by a lowerbound of $L$}
        \label{line:algo:one:prune:lb}      
        \State $\alg{CorelsEnum}(dp', k+1, L_*, \sig T\setminus\set{t}, \sig Y, \ell, S)$
        \Comment{Recursive call}
      \EndIf
    \EndFor
  \EndFor
  \State \kw{return} 
\end{algorithmic}
\end{algorithm}

In our algorithm, we employ some pruning techniques of \alg{CORELS} in a similar way to prune search of unnecessary subspaces as follows, where we attach comments to the corresponding part of \algoref{algo:reccorel:basic}: 
\vspace{-0.5\baselineskip}
\begin{enumerate}[(1)]
\item \name{Pruning based on minimum support}: it asserts that each rule must capture enough number of examples for the reliability of prediction.

\item \name{Pruning based on estimated lower bounds}: 
When invoking recursive call for a child $dp'$, if the lower bound function $LB$ does not satisfy $LB(dp', S)\le L_*$, prune all computation for $dp'$ and all of its descendants.
We use the \name{lower bound function} $LB(d, S)$ that is same to the empirical risk $L(h_d \mid S)$ except that all data that fall in the default rule are ignored~\cite{angelino:rudin:kdd2017corels} as in CORELS.

\item \name{Pruning based on symmetry}: 
If a range of consecutive rules $r_i, r_{i+1}, \dots, r_j$ in a rule list $d$, $1\le i< j \le k$, have the same labels $y_i = y_{i+1} = \dots = y_j$, any permutation of them does not change the prediction by $h_d$. Thus, we can keep some $r_\sigma$, $i\le \sigma\le j$, and discard the rest of them. 
\end{enumerate}


In spite of the inherent difference between 
\alg{CORELS}
with branch-and-bound search and \alg{CorelsEnum} with exhaustive search for all good rule lists, 
the above 
strategies (1)--(3) effectively prune the unnecessary subspaces of candidates. 
Let $M_\textrm{enum} \le |\sig Y|^{\ell}|\sig T|^{\ell-1}$ be the number of caldidate prefixes for \alg{CorelsEnum} to visit. We  show the following theorem.

\begin{theorem}\label{thm:corels:enum}
  \alg{CorelsEnum} of \cref{algo:reccorel:basic} enumerates all good rules with length $\le \ell$ on a data set $S \in (\sig X\times \sig Y)^N$ in 
  $t_\mathrm{enum} = O(M_\mathrm{enum} |S| )$ 
  time and 
  $s_\mathrm{enum} = O(|S| + |\sig T| + \ell^2)$ 
  space. 
\end{theorem}

\begin{proof}
  The time complexity follows that \alg{CorelsEnum} requires $O(|S|)$ time at each iteration to compute the objectives. The space complexity follows that the algorithm only keep at most $\ell$ rule lists with length $\le \ell$ on any branch of the search tree.  \qed 
\end{proof}

 A major advantage of \alg{CorelsEnum} is that \alg{CorelsEnum} has the polynomial space complexity in all inputs including $\ell$ independent of the number of solutions 
  $K \le |\sig T|^{\ell-1}|\sig Y|^{\ell}$, 
  while \alg{CorelsLawler} requires the space proportional to $K$, which may be exponential in $\ell$ in the worst case. 
  \alg{CorelsEnum} has amortized polynomial delay complexity, that is, it lists candidates in $O(|S|)$ time per candidate. 
  We remark that if the pruning strategy for $\alg{CORELS}$ effectively cuts candidates earlier on an input, it is possible that \alg{CorelsLawler} runs much faster than \alg{CorelsEnum} since the search space of the former is narrower than the latter. 
 





%% file: prob.tex

\section{Evaluation Criteria for Characterizing Rashomon Sets}
\label{sec:prob}

In this section, we introduce model criteria for analyzing the Rashomon set from the views of prediction multiplicity~\cite{Marx:ICML2020}, and fairness of prediction.~\cite{Hardt:NIPS2016}.

Several useful criteria have been proposed for characterizing some properties of a certain model class, such as interpretability~\cite{Fisher:JMLR2019}, multiplicity~\cite{Marx:ICML2020}, and fairness~\cite{Coston:ICML2021,Aivodji:NIPS2021}, through the lens of the Rashomon set. In particular, we focus on the \emph{predictive multiplicity} and \emph{unfairness range} described below. 
Note that these criteria can be easily computed once the Rashomon set $\sig R_{\varepsilon}(h_0 \mid S)$ is obtained. 

\subsection{Predictive Multiplicity}
\label{subsec:pred:multiple}

Marx {\em et al.}~\cite{Marx:ICML2020} have introduced the \emph{predictive multiplicity} as the ability of a prediction problem to admit competing models that assign conflicting predictions. 
Given a reference classifier $h_0$,
the predictive multiplicity is exhibited over the Rashomon set $\sig R_{\varepsilon}(h_0 \mid S)$ if there exists a classifier $h \in \sig R_{\varepsilon}(h_0 \mid S)$ such that $h({\vec x}) \neq h_0({\vec x})$ for some ${\vec x}$ in the dataset $S$. 
To measure the predictive multiplicity, \emph{ambiguity} and \emph{discrepancy} have been proposed~\cite{Marx:ICML2020}. 

\subsubsection{Ambiguity.}
Ambiguity represents the number of predictions by the reference classifier $h_0$ that can change over the set of competing classifiers $h \in \sig R_{\varepsilon}(h_0 \mid S)$. 
Formally, the ambiguity $\alpha_\varepsilon(h_0 \mid S)$ is defined by
\begin{align}
  \alpha_\varepsilon(h_0 \mid S)
  & := \frac{1}{N} {\sum}^{N}_{n=1} {\max}_{h \in \sig R_{\varepsilon}(h_0 \mid S)} \ind{h({\vec x}_n) \neq h_0({\vec x}_n)}
    \quad\in [0,1].
\end{align}
The ambiguity $\alpha_\varepsilon(h_0 \mid S)$ reflects the number of individuals ${\vec x}$ who could contest their assigned prediction $h_0({\vec x})$ by the deployed model $h_0$ since their predictions are determined depending on the model choice by the decision-makers~\cite{Marx:ICML2020}. 

\subsubsection{Discrepancy.}
Discrepancy represents the maximum number of predictions that can change if we switch the reference classifier $h_0$ with a competing classifier $h \in \sig R_{\varepsilon}(h_0 \mid S)$. 
Formally, the \emph{discrepancy} $\delta_\varepsilon(h_0 \mid S)$ is defined by
\begin{align}
  \delta_\varepsilon(h_0 \mid S)
  & := {\max}_{h \in \sig R_{\varepsilon}(h_0 \mid S)} \textrm{Dist}_\textrm{Hum}(h, h_0\mid S)
  \quad\in [0,1], 
\end{align}
where
\begin{math}
  \textrm{Dist}_\textrm{Hum}(h, h_0\mid S)
  := \frac{1}{N} {\sum}^{N}_{n=1} \ind{h({\vec x}_n) \not= h_0({\vec x}_n)}
  \in [0,1]
\end{math}
is the \name{normalized Hamming distance} between the vectors of the predictions by $h$ and $h_0$. 
Compared to the ambiguity, the discrepancy $\delta_\varepsilon(h_0 \mid S)$ reflects the number of the conflicting predictions $h({\vec x}_n) \not= h_0({\vec x}_n)$ by a single competing model $h \in \sig R_{\varepsilon}(h_0 \mid S)$~\cite{Marx:ICML2020}.

\subsection{Discrimination Scores and Unfairness Ranges}
\label{sec:prob:scores}
While Coston {\em et al.}~\cite{Coston:ICML2021} have proposed a framework that evaluate the fairness of the classifiers over the Rashomon set, Aïvodji {\em et al.}~\cite{Aivodji:NIPS2021} have pointed out that the Rashomon effect corresponds to the risk of \emph{fairwashing}, which is a malicious attack that rationales unfair complex models by interpretable and fair surrogate models~\cite{Aivodji:NIPS2021}. 
By motivating these studies, we introduce the \emph{unfairness range} to evaluate the fairness over the Rashomon set $\sig R_{\varepsilon}(h_0 \mid S)$. 

Let $z_n \in \set{0,1}$ be a \emph{sensitive attribute} (e.g., gender or race) with respect to the $n$-th example $({\vec x}_n, y_n)$ in a dataset $S$. 
To evaluate the fairness of a classifier $h$ with respect to the sensitive attribute $z$, we focus on \emph{demographic parity (DP)}~\cite{Calders:ICDMWS2009} and \emph{equal opportunity~(EO)}~\cite{Hardt:NIPS2016}, which are major discrimination criteria based on statistical parity.
The DP and EO scores of $h$ on $S$ are defined as:
\begin{math}
    \mathrm{DP}(h \mid S) := \hat{P}(h({\vec x})=1 \mid z=1) - \hat{P}(h({\vec x})=1 \mid z=0),  
    {\mathrm{EO}}(h \mid S) := \hat{P}(h({\vec x})=1 \mid y=1, z=1) - \hat{P}(h({\vec x})=1 \mid y=1, z=0), 
\end{math}
where $\hat{P}$ is the empirical probability over the joint distribution on $y$, $z$, and $h({\vec x})$ of $S$. 

Let $D \in \set{\mathrm{DP}, \mathrm{EO}}$ be any discremination score. 
We introduce the \name{unfairness range} of the Rashomon set $\sig R_{\varepsilon}(h_0 \mid S)$, 
denoted  $\gamma_\varepsilon^\mathit{D}$, 
as an approximation of the distribution of $D$ for the models in $\sig R_{\varepsilon}$. Formally, the unfairness range is the interval 
\begin{math}
    \gamma_\varepsilon^\mathit{D}(h_0 \mid S) := \left[\; {\min}_{h} D(h \mid S), {\max}_{h} D(h \mid S) \;\right]  \subseteq [-1, +1],
\end{math}
where $h$ ranges over $\sig R_{\eps}(h_0 \mid S)$. 
Since we can exactly compute the Rashomon set $\sig R := \sig{R}_\eps(h_0 \mid S)$ in $t_{enum}$ by using \alg{CorelsEnum} proposed in \cref{sec:algo:proposed},  now we can compute the range $ \gamma_\varepsilon^\mathit{D}(h_0 \mid S)$  in linear time in $t_{enum} + |\sig R|$ by scanning~$\sig R$. 




%% file: exp.tex

\section{Experiments}\label{sec:exp}
\begin{figure}[t]
    \centering
     \includegraphics[width=0.8\columnwidth]{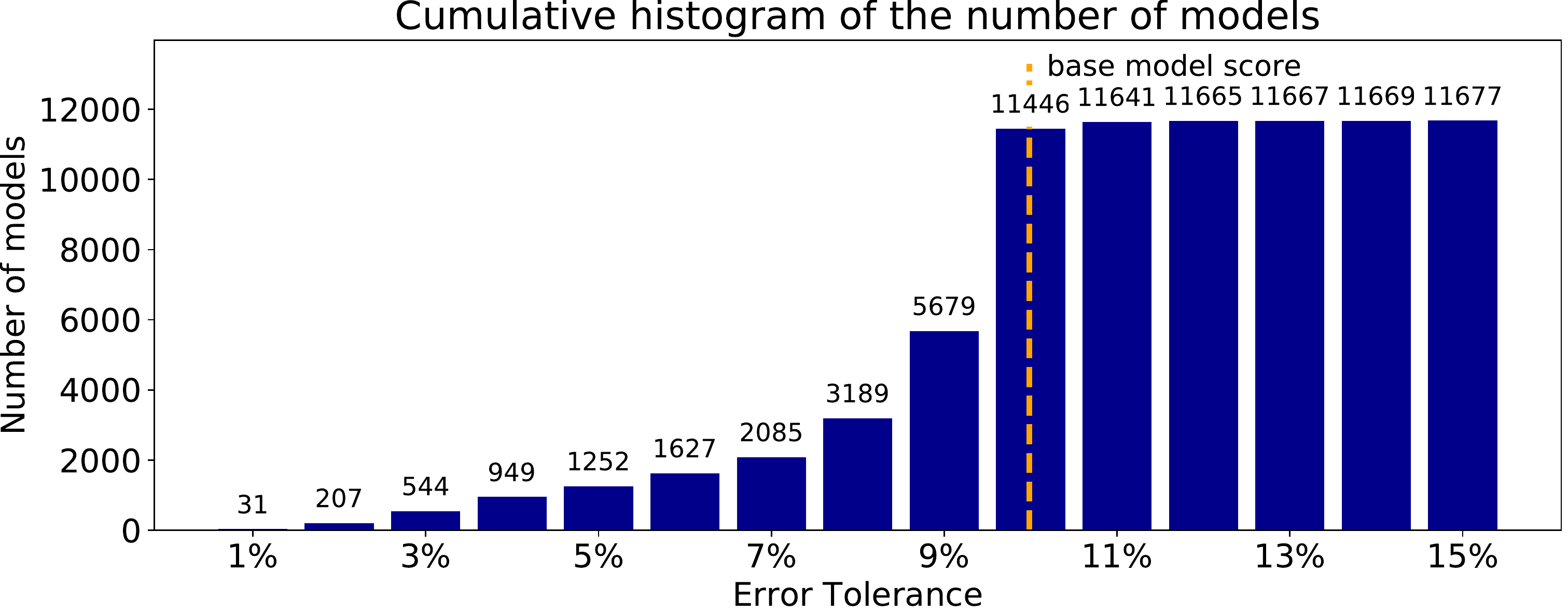}
    \caption{Cumulative histogram of the number of the models in the Rashomon set for each value of the error tolerance $\varepsilon$ from $1\%$ to $15\%$. }
    \label{fig:numberofmodels}
    \vspace{-0.5\baselineskip}
\end{figure}

\begin{table}[t]
 \caption{Results of execution time and the number of models found on COMPAS dataset by the existing method (\alg{CorelsLawler}) and our proposed method (\alg{CorelsEnum}) within around 6,000 seconds. The existing method was stopped at $K=40$ by timeout. }
 \label{table:compare}
 \centering
  \begin{tabular}{l|ccc}
   \hline
    & Existing method & Proposed method \\
   \hline \hline
   Run time (s) & $6021$ & $1058$ \\
   Memory(MB) & 209.3 & 202.4 \\
   Number of models & $40$ & $23354$ \\
   \hline
  \end{tabular}
\end{table}

\begin{figure}[t]
    \centering
     \includegraphics[width=0.9\columnwidth]{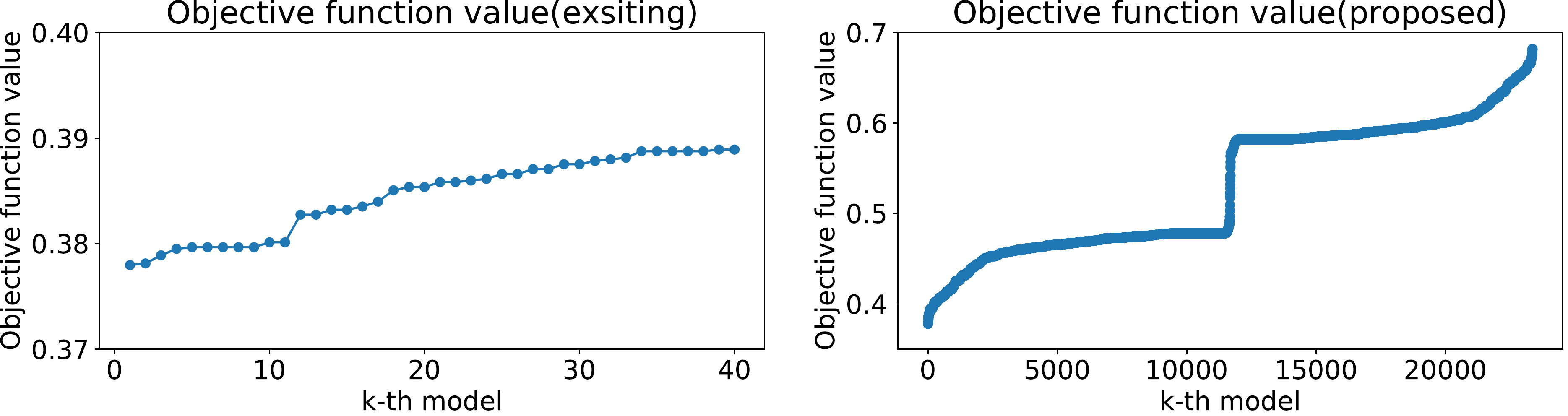}
    \caption{The objective value (training error plus $\lambda$ times the rule list length) against the rank $k$ of a rule list on the COMPAS dataset for existing and proposed methods.}
    \label{fig:compare}
    \vspace{-0.5\baselineskip}
\end{figure}
\begin{figure}[t]
    \centering
     \includegraphics[width=\columnwidth]{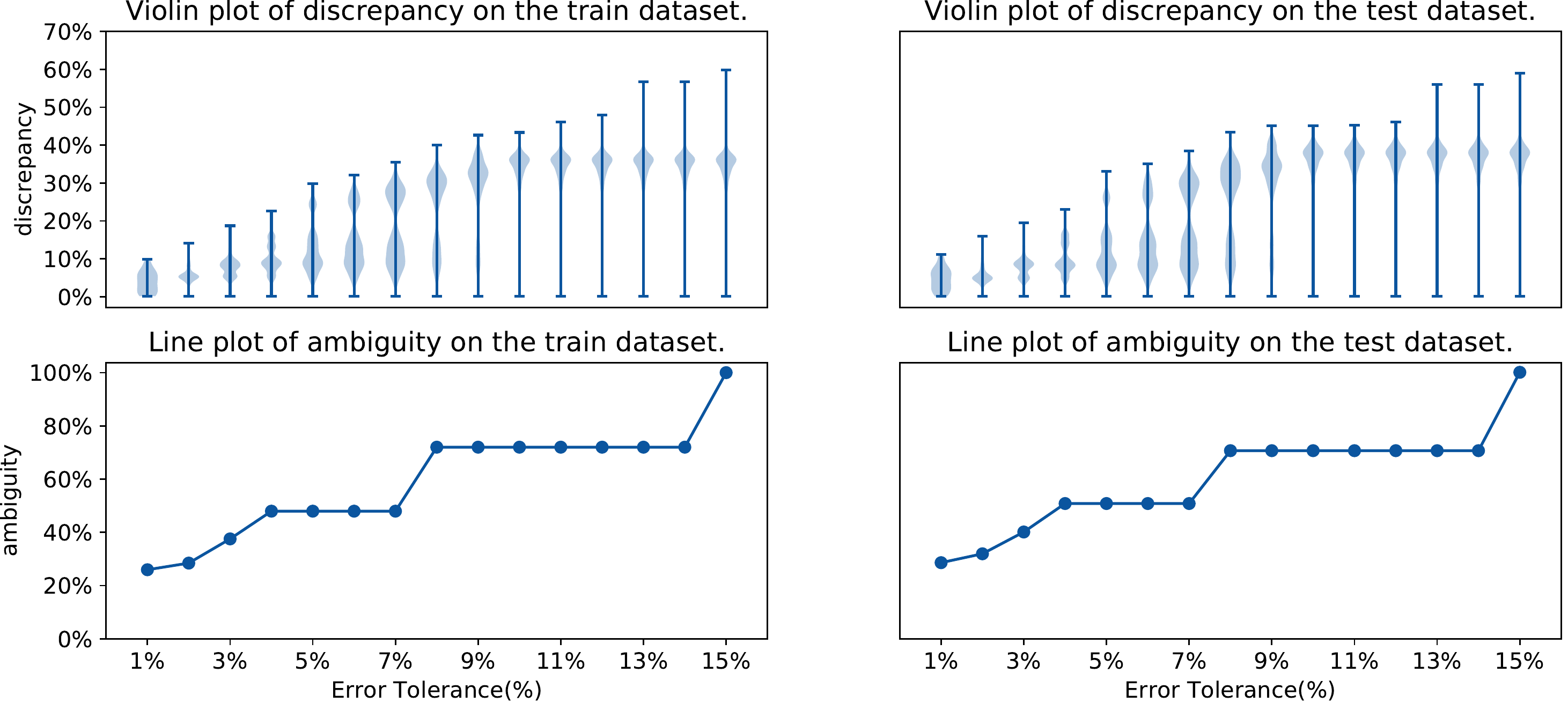}
    \caption{Predictive multiplicity of discovered rule lists in the Rashomon set on the COMPAS dataset. The violin plots (above) show the distribution of discrepancy and the line plots (below) show the ambiguity of rule lists over the Rashomon set.}
    \label{fig:predictivemultiplicity}
    \vspace{-0.5\baselineskip}
\end{figure}

\begin{figure}[t]
    \centering
     \includegraphics[width=0.6\columnwidth]{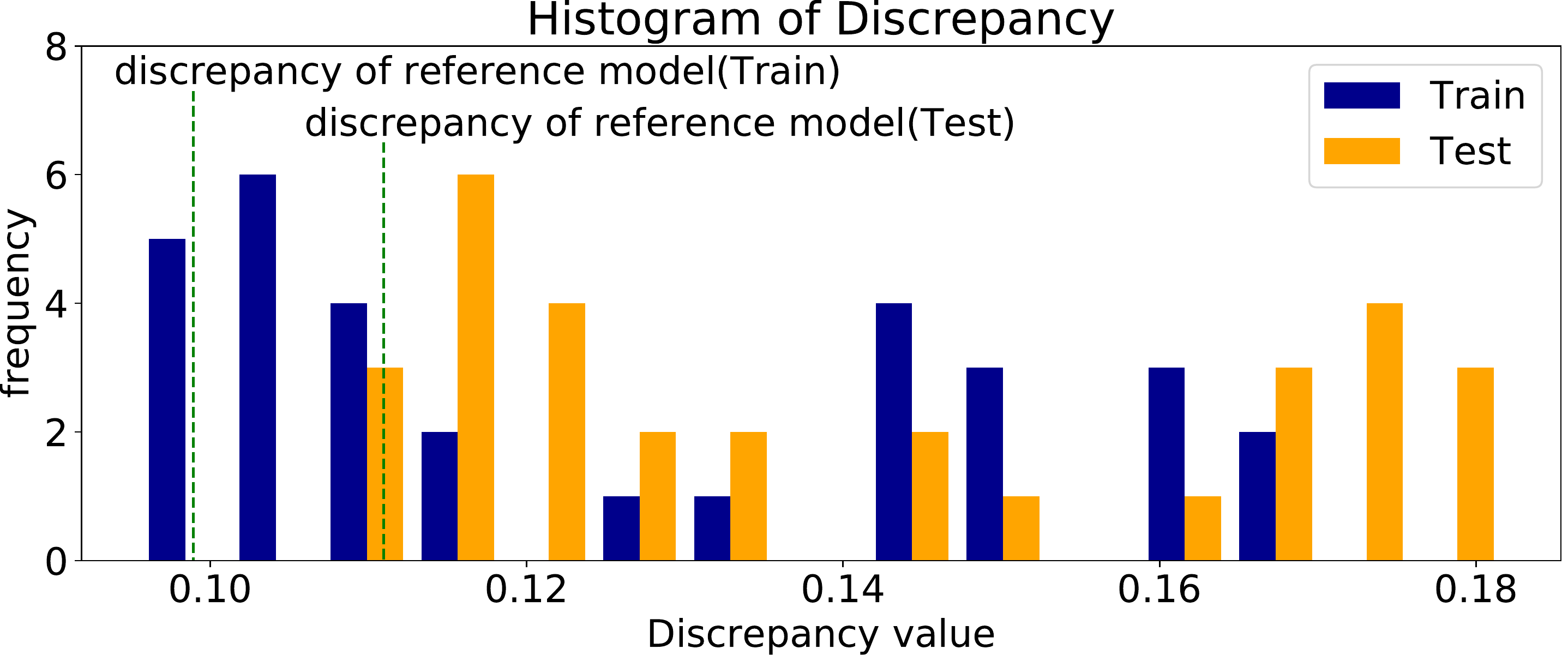}
    \caption{The histograms of the discrepancies $\delta_\eps$ of discovered rule lists in the Rashomon set with error tolerance $\eps = 1\%$ on the COMPAS dataset, where the blue and yellow histograms  show the frequencies in the training and test data sets, respectively. }
    \label{fig:radius}
    \vspace{-0.5\baselineskip}
\end{figure}

\begin{figure}[t]
    \centering
     \includegraphics[width=\columnwidth]{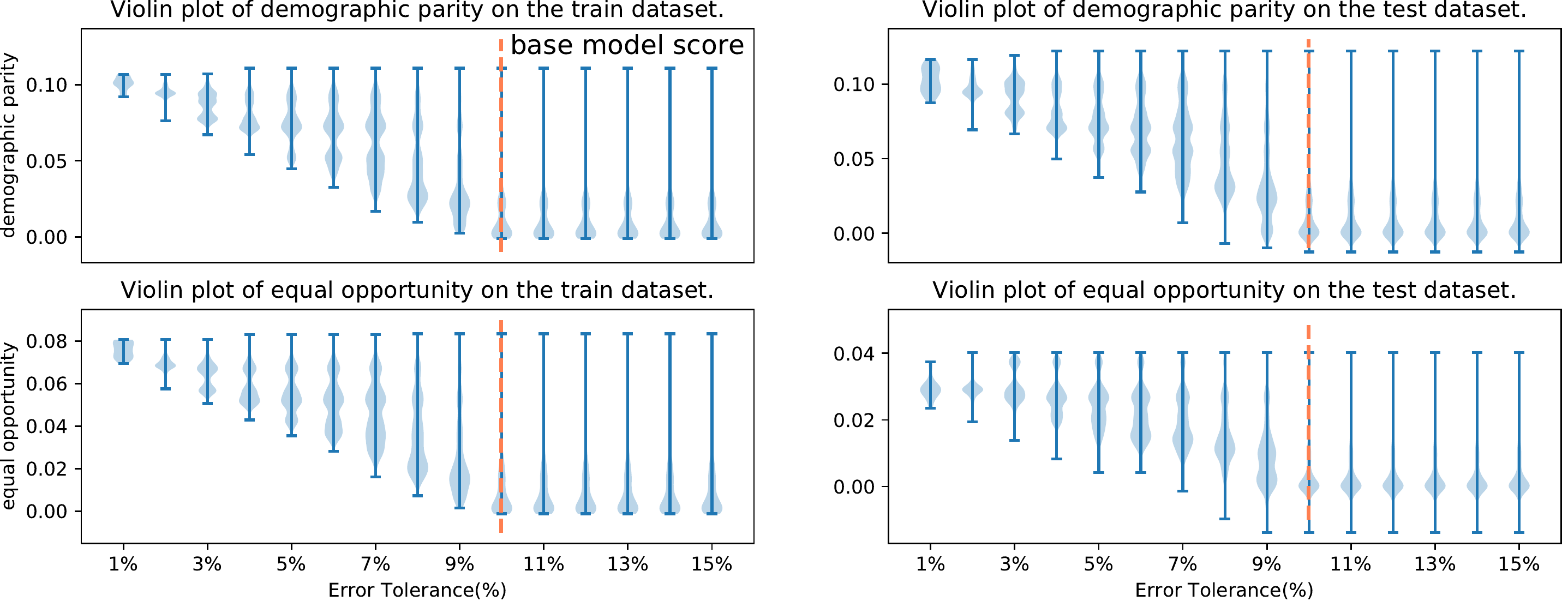}
    \caption{
    Distributions of discrimination scores with respect to demographic parity ($\mathrm{DP}$, upper) and  equal opportunity ($\mathrm{EO}$, lower) on the COMPAS dataset. Here, the violin plots show the frequencies of models with a certain score, while the error bar shows the unfairness range with $\mathrm{DP}$ and $\mathrm{EO}$. The score for the base model is shown in red dashed lines).}
    \label{fig:unfairnessrange}
    \vspace{-0.5\baselineskip}
\end{figure}

In this section, 
we analyze the class of rule lists on the COMPAS dataset~\cite{compas:2016} through the lens of the Rashomon effect using our proposed algorithm.

\subsection{Experimental Setting}

\textbf{Datasets}. We used COMPAS dataset~\cite{compas:2016} 
for the task of criminal decision, 
which comprises $20$ categorical attributes of individual people, relating their criminal history, with a total of $6,489$ training examples ($90\%$) $S$ and $721$ test examples ($10\%$) $S'$. 
The task is binary classification, where the positive category $y = 1$ indicates that the individual recidivates within two years. 
The sensitive attributes $z$ represent the race of the individuals. 
\textbf{Programs}. We implemented \alg{CorelsLawler} and \alg{CorelsEnum} (\cref{sec:algo}) in Python 3.7 with \texttt{numpy} package. 
All the experiments were conducted on 64-bit macOS Big Sur 11.2.3 with Intel Core i9 2.4GHz CPU and 32GB Memory.
We used the 
libraries: \texttt{pandas} for preprocessing, and \texttt{matplotlib}, \texttt{pyplot.violinplot} 
for 
charts. 

\textbf{Setting}. Throughout this paper, we used the following setting for model parameters. A label set is $\sig Y = \set{\id{Yes}, \id{No}}$, and the maximum length of rule lists is $\ell = 3$. 
We used the vocabulary $\sig T\subseteq \sig T_\textrm{corels}$ of 64 terms selected from the set $\sig T_\textrm{corels}$ of all 155 terms in the github repository of \alg{CORELS}~\cite{angelino:rudin:kdd2017corels} so that a term  $t$ is selected if and only if it evaluates true on at least half of the positive examples%
\footnote{This was because we were interested in characterizing the positive category as in~\cite{hara:ishihata:aaai2018rulemodels}.}
as with previous studies~\cite{hara:ishihata:aaai2018rulemodels}. 
Consequently, we obtain a candidate space of size 
$M 
= |\sig T|^{\ell-1}|\sig Y|^{\ell} 
= 64^2 2^3 
= 32,798
$.
We first obtained a reference classifier $h_0$ by \alg{CORELS} on the training dataset $S$, and then computed the Rashomon set $\mathcal{R}_{\varepsilon} = \mathcal{R}_{\varepsilon}(h_0 \mid S)$ by $\alg{CorelsEnum}$.
We computed 
$\mathcal{R}_{\varepsilon}$ 
by varying the error tolerance $\varepsilon$ from $1\%$ to $15\%$, 
and analyzed 
its 
properties
for each $\varepsilon$. 

\subsection{Experimental Results}

\subsubsection{The Numbers of Good Rule Lists by Varying the Error Tolerance}

In \cref{fig:numberofmodels},  we show the number of the models in the Rashomon set by varying the error tolerance $\varepsilon$ from $1\%$ to $15\%$. In the figure, the reference model $h_0$ locates at $\varepsilon = 1\%$, which amounts to training error $ 34.8 \%$, while the baseline model $h_*$, which is such a constant rule list that always outputs $y=0$ for any input ${\vec x}$ locates at $\varepsilon = 10\%$, which amounts to the training error $44.8 \%$, i.e., the ratio of the examples with $y=0$ in the training dataset $S$.
From \cref{fig:numberofmodels}, we can see that the total number of models in $\mathcal{R}_{\varepsilon}(h_0 \mid S)$ increases rapidly between the error tolerance $\varepsilon$ of $9\%$ and $10\%$, and almost saturates after $\varepsilon$ exceeds $10\%$. 
For example, the Rashomon set with $\varepsilon=9\%$ (resp.\ $\varepsilon=10\%$) contained $5679$ (resp.\ $11446$) rule lists.  
This is because the Rashomon sets with $\varepsilon \geq 10 \%$ included exponentially many rule lists as accurate as the baseline model $h_*$ in the number of candidate terms in $\mathcal{T}$. 

\subsubsection{Comparison of the Existing and the Proposed Algorithms}
Next, we compared 
the existing method (\alg{CorelsLawler}) 
in Sec.~\ref{sec:algo:lawler:with:corels} and 
our proposed method (\alg{CorelsEnum}) in Sec.~\ref{sec:algo:proposed} in terms of running time and memory. 
We ran experiments for finding good rule lists in the objective function $R_{\lambda}$ with
parameter $\lambda = 0.015$ for both algorithms  within around 6,000 seconds. 

\cref{table:compare} shows the comparison of the running time and memory usage of both algorithms within 6,000 seconds. 
We see that without limit of the error tolerance $\varepsilon$, the proposed algorithm \alg{CorelsEnum} enumerated all 23354 models including all good models for any $\varepsilon \ge 0$, while the existing method was stopped at $K=40$ by timeout of $6,000$ seconds after finding top-40 good models. 
From these results, we observed that the proposed 
\alg{CorelsEnum} was about $5.7$ times faster than the existing 
\alg{CorelsLawler}. 
\cref{fig:compare} shows the objective function value against the rank of the models.  
For the top-40 models, we confirmed that both algorithms successfully found models with the same value of the objective function.

\subsubsection{Predictive Multiplicity}
\label{sec:exp:multiplicities}
Next, we examine the predictive multiplicity of the Rashomon set on the COMPAS dataset. 
\cref{fig:predictivemultiplicity} shows the results on 
the discrepancy $\delta_\eps = \delta_\eps(h_0 \mid S)$ and 
ambiguity $\alpha_{\eps} = \alpha_{\eps}(h_0 \mid S)$ of the Rashomon set $\mathcal{R}_{\eps} = \mathcal{R}_{\eps}(h_0 \mid S)$ on the training dataset $S$ for each $\varepsilon$. 
From \cref{fig:predictivemultiplicity}, we observed that the values of 
$\delta_{\eps}$ and $\alpha_{\eps}$
monotonically increased as $\eps$ increased. 
For example, the value of discrepancy (resp.\ ambiguity) with $\varepsilon = 1\%$ was $\delta_\eps = 11\%$ (resp.\ $\alpha_\eps = 29\%$). 
These results imply that $11\%$ of predictions can be changed by switching the reference classifier $h_0$ with a classifier 
$h \in \mathcal{R}_{\varepsilon}$ 
that is only $1\%$ less accurate, and that $29\%$ of individuals are assigned conflicting predictions by at least one classifier $h \in \mathcal{R}_{\eps}$ with the error tolerance $1\%$~\cite{Marx:ICML2020}. 
We also measured 
$\delta_{\eps}$ for all good rule lists $h$ in $\mathcal{R}_{\eps}$. 
\cref{fig:radius} shows the histogram of these values with $\varepsilon = 1\%$ on the training datasetg $S$ and test dataset $S'$. 
From \cref{fig:radius}, we can see that there were rule lists that achieved lower discrepancy than that of the reference classifier $h_0$. 
It suggests that we can obtain another reference classifier with lower discrepancy than $h_0$ by exhaustive search of the Rashomon set 
$\mathcal{R}_{\varepsilon}$. 

\subsubsection{Unfairness Range}

Finally, \cref{fig:unfairnessrange} shows the distribution of discrimination scores demographic parity ($\mathrm{DP}$) and equal opportunity ($\mathrm{EO}$) on rule lists 
in the Rashomon set. 
In the figure, we can clearly see the trade-off between the empirical risk $L(h \mid S')$ and the minimum discrimination scores by the lower ends of violin plots, which is consistent with existing theoretical results~\cite{Hardt:NIPS2016}.  For example, 
we have a higher discrimination value of  $\mathrm{DP}=0.10$ in the higher accuracy case with error tolerance $\eps = 1\%$, while we can have a lower and better value of $\mathrm{DP}=0.02$ in the lower accuracy case with error tolerance $\eps = 7\%$. After error tolerance $\eps \ge 10\%$ of the trivial, constant learner, we see that the distribution becomes stable, and most rule lists in the population hold the lowest $\mathrm{DP} = 0.02$. 
Furthermore, we can see that the rule lists are concentrated to a few clusters, in the violin plots for $\varepsilon$ from $1\%$ to $8\%$, indicating existence of a few subgroups of good rule lists that behave similarly in their syntax and predictions. 

%% file: concl.tex
\section{Conclusion}
In this paper, we studied efficient computation of all good models in the Rashomon set for the class of rule lists. 
By extending a state-of-the-art algorithm \textit{CORELS} for a globally optimal rule list, we proposed an exact algorithm  \alg{CorelsEnum} for enumerating all the rule lists in the Rashomon set. 
To evaluate the usefulness of \alg{CorelsEnum}, we conducted experiments on the COMPAS dataset, and analyzed the computed Rashomon set of the rule lists from the perspectives of predictive multiplicity and fairness. 

In future work, we plan to 
conduct experiments on other real datasets and with larger values of $\ell\ge 4$. It is also interesting to extend our algorithm to other rule models, such as decision trees of bounded size. 

\medskip
\textbf{Acknowledgement}. The authors would like to thank anonymous referees for their valuable comments that improves the quality of this paper. 
This work was partly supported grants from Grant-in-Aid for JSPS Research Fellow 20J20654, and Grant-in-Aid for Scientific Research(A) 20H0059. 



%% file: main.bbl
\begin{thebibliography}{10}
\providecommand{\url}[1]{\texttt{#1}}
\providecommand{\urlprefix}{URL }
\providecommand{\doi}[1]{https://doi.org/#1}

\bibitem{Aivodji:NIPS2021}
Aivodji, U., Arai, H., Gambs, S., Hara, S.: Characterizing the risk of
  fairwashing. In: Proc. NeurIPS 2021, to appear (2021)

\bibitem{angelino:rudin:kdd2017corels}
Angelino, E., Larus-Stone, N., Alabi, D., Seltzer, M., Rudin, C.: Learning
  certifiably optimal rule lists. In: Proc. KDD 2017. p. 35^^e2^^80^^9344
  (2017)

\bibitem{compas:2016}
Angwin, J., Larson, J., Mattu, S., Kirchner, L.: {Machine Bias}. ProPublica
  (2016)

\bibitem{breiman:statsci2001twocultures}
Breiman, L.: {Statistical Modeling: The Two Cultures (with comments and a
  rejoinder by the author)}. Stat. Sci.  \textbf{16}(3),  199 -- 231 (2001)

\bibitem{Calders:ICDMWS2009}
Calders, T., Kamiran, F., Pechenizkiy, M.: Building classifiers with
  independency constraints. In: Proc. ICDM Workshops 2009. pp. 13--18 (2009)

\bibitem{Coston:ICML2021}
Coston, A., Rambachan, A., Chouldechova, A.: Characterizing fairness over the
  set of good models under selective labels. In: Proc. ICML 2021. pp.
  2144--2155 (2021)

\bibitem{Fisher:JMLR2019}
Fisher, A., Rudin, C., Dominici, F.: All models are wrong, but many are useful:
  Learning a variable's importance by studying an entire class of prediction
  models simultaneously. J. Mach. Learn. Res.  \textbf{20}(177),  1--81 (2019)

\bibitem{Guidotti:CSUR}
Guidotti, R., Monreale, A., Ruggieri, S., Turini, F., Giannotti, F., Pedreschi,
  D.: A survey of methods for explaining black box models. CSUR.
  \textbf{51}(5),  1--42 (2018)

\bibitem{Han:kamber:pei2011dmbook}
Han, J., Kamber, M., Pei, J.: Data Mining: Concepts and Techniques. Morgan
  Kaufmann, 3rd edn. (2011)

\bibitem{Hancox-Li:FAT*2020}
Hancox-Li, L.: Robustness in machine learning explanations: Does it matter? In:
  Proc. FAT* 2020. pp. 640--647 (2020)

\bibitem{hara:ishihata:aaai2018rulemodels}
Hara, S., Ishihata, M.: Approximate and exact enumeration of rule models. In:
  Proc. AAAI 2018. pp. 3157--3164 (2018)

\bibitem{Hardt:NIPS2016}
Hardt, M., Price, E., Srebro, N.: Equality of opportunity in supervised
  learning. In: Proc. NeurIPS 2016. pp. 3323--3331 (2016)

\bibitem{hastie2001eslbook}
Hastie, T., Tibshirani, R., Friedman, J.: The Elements of Statistical Learning.
  Springer Series in Statistics, Springer (2001)

\bibitem{Lakkaraju:KDD2016}
Lakkaraju, H., Bach, S.H., Leskovec, J.: Interpretable decision sets: A joint
  framework for description and prediction. In: Proc. KDD 2016. pp. 1675--1684
  (2016)

\bibitem{Lawler:MS1972}
Lawler, E.L.: A procedure for computing the $k$ best solutions to discrete
  optimization problems and its application to the shortest path problem.
  Manag. Sci.  \textbf{18}(7),  401--405 (1972)

\bibitem{Marx:ICML2020}
Marx, C., Calmon, F., Ustun, B.: Predictive multiplicity in classification. In:
  Proc. ICML 2020. pp. 6765--6774 (2020)

\bibitem{Rudin:NMI2019}
Rudin, C.: Stop explaining black box machine learning models for high stakes
  decisions and use interpretable models instead. Nat. Mach. Intell.
  \textbf{1},  206--215 (2019)

\bibitem{semenova:rudin:parr:arxiv2019rashomon}
Semenova, L., Rudin, C., Parr, R.: A study in rashomon curves and volumes: A
  new perspective on generalization and model simplicity in machine learning.
  arXiv preprint, arXiv:1908.01755  (2019)

\bibitem{uno2004lcmver2}
Uno, T., Kiyomi, M., Arimura, H., et~al.: Lcm ver. 2: Efficient mining
  algorithms for frequent/closed/maximal itemsets. In: Proc.~FIMI 2004 (2004)

\bibitem{Wang:AISTATS2015}
Wang, F., Rudin, C.: {Falling Rule Lists}. In: Proc. AISTATS 2015. pp.
  1013--1022 (2015)

\end{thebibliography}
